\documentclass[runningheads]{llncs}
\usepackage[T1]{fontenc}
\usepackage{graphicx}
\usepackage{algorithm}
\usepackage[noend]{algorithmic}

\usepackage{subfigure}
\usepackage{amsmath}

\usepackage{amssymb}
\usepackage{mathrsfs}

\newcommand{\beq}{\begin{equation}}
\newcommand{\eeq}{\end{equation}}
\newcommand{\bal}{\begin{aligned}}
\newcommand{\eal}{\end{aligned}}

\newcommand{\nodeSet}{\mathcal{N}}
\newcommand{\arcSet}{\mathcal{A}}
\newcommand{\layerSet}{\mathcal{N}}
\newcommand{\layer}{\ell}
\newcommand{\rootnode}{{\bf r}}
\newcommand{\terminalnode}{{\bf t}}
\newcommand{\arcvalue}{w}
\newcommand{\arcdomain}{d}
\newcommand{\pathvalue}{w}

\newcommand{\pathSet}{\mathcal{P}}
\newcommand{\solutionSet}{\mathrm{Sol}}
\newcommand{\bddwidth}{W}
\newcommand{\optproblem}{\mathscr{P}}
\newcommand{\DD}{\mathcal{D}}
\newcommand{\statespace}{\mathcal{S}}
\newcommand{\state}{S}
\newcommand{\transition}{f}
\newcommand{\transitioncost}{g}
\newcommand{\feasibledecisions}{X}

\newcommand{\RN}[1]{%
  \textup{\uppercase\expandafter{\romannumeral#1}}%
}

\begin{document}
\title{A Clustering-Based Variable Ordering Framework for Relaxed Decision Diagrams for Maximum Weighted Independent Set Problem}
\titlerunning{A Clustering-Based Variable Ordering Framework for MWISP}
%
\author{Mohsen Nafar\inst{1}\orcidID{0000-0002-0895-2837} \and
Michael Römer\inst{2}\orcidID{0000-0001-8369-7939} \and
Lin Xie\inst{1}\orcidID{0000-0002-3168-4922}}
\authorrunning{M. Nafar et al.}
%
\institute{Brandenburg University of Technology, Cottbus, Germany 
\\
\email{\{mohsen.nafar, lin.xie\}@b-tu.de}\\
\and
Bielefeld University, Bielefeld, Germany\\
\email{michael.roemer@uni-bielefeld.de}}
\maketitle              
\begin{abstract}
Efficient exact algorithms for Discrete Optimization (DO) rely heavily on strong primal and dual bounds. Relaxed Decision Diagrams (DDs) provide a versatile mechanism for deriving such dual bounds by compactly over-approximating the solution space through node merging. However, the quality of these relaxed diagrams, i.e. the tightness of the resulting dual bounds, depends critically on the variable ordering and the merging decisions executed during compilation. While dynamic variable ordering heuristics effectively tighten bounds, they often incur computational overhead when evaluated globally across the entire variable set. To mitigate this trade-off, this work introduces a novel clustering-based framework for variable ordering. Instead of applying dynamic ordering heuristics to the full set of unfixed variables, we first partition variables into clusters. We then leverage this structural decomposition to guide the ordering process, significantly reducing the heuristic's search space. Within this framework, we investigate two distinct strategies: Cluster-to-Cluster, which processes clusters sequentially using problem-specific aggregate criteria (such as cumulative vertex weights in the Maximum Weighted Independent Set Problem (MWISP)), and Pick-and-Sort, which iteratively selects and sorts representative variables from each cluster to balance local diversity with heuristic guidance. Later on, developing some theoretical results on the growth of the size of DDs for MWISP we propose two different policies for setting the number of clusters within the proposed framework. We embed these strategies into a DD-based branch-and-bound algorithm and evaluate them on the MWISP. Across benchmark instances, the proposed methodology consistently reduces computational costs compared to standard dynamic variable ordering baseline.

\end{abstract}

\section{Introduction}

Using decision diagrams (DDs), one can represent the solution space of Discrete Optimization (DO) problems in a compact way as a layered directed graph where every layer corresponds to a decision variable. Such a DD representation can be obtained from a Dynamic Programming (DP) formulation of the DO problem, making DDs a highly generic tool for DO.  For a recent survey on advances in DDs for optimization, see \cite{Castro2022decision}, and an introduction to DDs for optimization can be found in \cite{van2024introduction}. 

Although it is possible to create an exact DD that represents the complete solution space for the problem, the size of this type of DD grows exponentially. By restricting the width of the DDs, i.e. imposing a given maximum width \(\bddwidth\), it is possible to create approximate DDs, i.e. restricted and relaxed DDs, that provide primal and dual bounds which can be used in exact solution approaches such as DD-based branch-and-bound \cite{Bergman2016Discrete}. 
In restricted DDs, some feasible solutions are removed when the width of a layer exceeds \(\bddwidth\). In relaxed DDs, nodes associated with non-equivalent states are merged. 
Bergmann et al. \cite{Bergman2016DDbook} demonstrate the use of relaxed and restricted DDs within an exact Branch-and-bound Bethod and showed excellent performance achievements for the Maximum Independent Set Problem (MISP), the Maximum Cut Problem and the 2-satisfiability problem.  Rudich et al. \cite{rudich2022peel,rudich2023improved} proposed the Peel-and-Bound (PnB) method that avoids many repetitive computations during the Branch-and-bound process. 

The fact that DD-based solution approaches can be very effective and at the same time can be applied to very different DO problems as long as they are formulated as DPs led to the development of several generic high-performance solvers such as DDO \cite{gillard2021ddo} and CODD \cite{michel2024codd} which allow using DD-based solution techniques in a declarative model-and-solve fashion. Similarly, Kuroiwa und Beck \cite{kuroiwa2023domain} proposed Domain-Independent Dynamic Programming (DIDP) which uses various state-of-the-art DP and tree search algorithms for solving DO problems stated as DPs.

The effectiveness of DD-based branch-and-bound algorithms depends on the quality of the bounds of the approximate DDs, which can be constructed layer by layer using a top-down approach. In this approach, two heuristic decisions affect the quality of the primal and dual bounds: Variable ordering, which considers the order of variables in the top-down compilation, and node selection, which determines which nodes should be removed or merged.

Various variable ordering methods have been proposed, ranging from classical heuristics to machine learning (ML)-based approaches. For instance, the heuristic called ``MIN'' that has become one of the standard variable ordering heuristics for the MISP in the DD literature in which the variable that appears least frequently across the current set of states \cite{Bergman2016DDbook}. Another example is called ``CDS'', it was proposed for MISP \cite{nafar2024strengthening} and uses graph-theoretical properties of subgraphs induced on the variables included in states. 
Beyond these conventional techniques, recent research has explored the integration of ML into the variable ordering process. These ML-based methods aim to learn effective ordering policies from data, and have demonstrated promising results on MISP, as shown in works such as \cite{karahalios2022variable}, \cite{Parjadis2021Improving}, and \cite{Cappart2019Improving}. By adapting variable selection in response to learned patterns or features of the problem instance, these approaches offer a complement to traditional strategies.

\paragraph{Contribution.} In this work, we present a novel clustering-based framework for ordering variables for the top-down compilation of relaxed DDs. In this framework, we first group variables using a clustering algorithm, subsequently applying one of two distinct compilation strategies. The first, \textit{Cluster-by-Cluster (CbC)}, sorts the clusters based on problem-specific aggregate criteria (e.g., the cumulative vertex weights in the MWISP) and applies a dynamic variable ordering heuristic sequentially within each cluster. The second, \textit{Pick-and-Sort (PaS)}, iteratively selects one representative variable from each cluster based on a variable ordering heuristic, then sorts the selected set of variables before insertion into the DD. The sorting of variables in PaS is based on either (i) a problem-specific criterion (PaS), or (ii) the result of the variable ordering heuristic that is used within the framework (PaS-VO). We then propose two different policies for setting the number of clusters within our proposed framework. The proposed policies are motivated by a theoretical study on the growth of the size of DDs for MWISP in this paper. We integrate this framework into a standard DD-based branch-and-bound algorithm and evaluate it on the MWISP. Computational results demonstrate that the proposed methodology consistently outperforms standard dynamic ordering, yielding substantial reductions in overall solution time for almost all instances, i.e. graphs with densities ranging from 0.9 to 0.2.

\section{Exact and Approximate Decision Diagrams}

A DD $\DD=(\nodeSet, \arcSet)$ is a directed acyclic layered graph with a  set of nodes  $\nodeSet$ and a set of arcs $\arcSet$.  $\nodeSet$ is partitioned into $n+1$ layers $\layerSet_1, \dots, \layerSet_{n+1}$, where 
$\layerSet_1 = \{\rootnode\}$ and $\layerSet_{n+1} = \{\terminalnode\}$ for a \textit{root} $\rootnode$ and a \textit{terminal} $\terminalnode$. 
Each path from  $\rootnode$  to $\terminalnode$ in $\DD$ represents a solution to a discrete optimization problem $\optproblem$  with a maximization objective $z$ and $n$ decision variables $x_1,\cdots,x_n \in \{0, 1\}$. Each arc $a=(u,u')$ connects  nodes of two consecutive layers $\layer(u), \layer(u') = \layer(u)+1$  and is associated with a decision $\arcdomain(a)$ representing the assignment $x_{\layer(u)} = \arcdomain(a)$. This means that a path $p=(a_1,\dots,a_n)$ starting from $\rootnode$ and ending at $\terminalnode$ represents the solution $x(p) = (\arcdomain(a_1),\dots,\arcdomain(a_n))$. We denote the set of all $\rootnode$-$\terminalnode$ paths with $\pathSet$, and we refer to the solutions to $\optproblem$ represented by $\pathSet$ with $\solutionSet(\DD)$. Moreover, each arc $a$ has length $\arcvalue(a)$ and  $\sum_{i=1}^n \arcvalue(a_i)$ provides the length $\pathvalue(p)$ of path $p$.

 $\DD$ is called an exact DD if $\solutionSet(\DD) = \solutionSet(\optproblem)$ and if for each path $p \in \pathSet$ we have $\pathvalue(p) = z(x(p))$; then a longest path in $\DD$ forms an optimal solution to $\optproblem$. However, such an exact DD grows exponentially with the instance size of the DO problem under consideration, and thus, effective DD-based solution approaches rely on so-called approximate DDs that can be used to obtain upper or lower bounds for the solutions of $\optproblem$. There are two types of approximate DDs: in a \textit{restricted} DD $\DD$, which provides a lower bound to $\optproblem$, one only considers promising nodes and arcs, meaning that $\solutionSet(\DD) \subseteq \solutionSet(\optproblem)$. The second type of approximate DD, which is the one we focus on in this paper, is the \textit{relaxed} DD providing an upper bound: In a relaxed DD, we have $\solutionSet(\DD) \supseteq \solutionSet(\optproblem)$, that is, the set of paths may contain paths associated with infeasible solutions to $\optproblem$. Regarding the objective function value, every path in a relaxed DD needs to satisfy $\pathvalue(p) \geq z(x(p))$. In both types of approximate DDs, one usually limits the DD size by enforcing  a maximum width $\bddwidth$ for each layer by removing nodes (in a restricted DD) or merging nodes (in a relaxed DD).

An important way to create an exact DD relies on a Dynamic Programming (DP) formulation of $\optproblem$ that is used to compile the DD in a top-down fashion. To do so, every node $u$ is associated with a state $\state_u$ and every arc $a$ is associated with a state transition induced by the decision $d(a)$ associated with $a$. $\state_u$ is an element of the state space $\statespace$; the state  $\state_{\rootnode}$ associated with  $\rootnode$ is the so-called \textit{initial state}. The state $\state_v$ of the target node $v$ of the arc depends on the state $S_u$ of the arc's source node as well as on $d$ and is computed by the state-transition function $\transition(\state_u,d)$. The contribution to the objective function induced by a decision is computed by a reward function $\transitioncost(\state_u,d)$. Finally, the set of out-arcs of a node $u$ is determined by the set of feasible decisions $\feasibledecisions(\state_u)$ given state $\state_u$. The top-down compilation then proceeds layer by layer until it reaches the layer $\layerSet_{n}$; all arcs emanating from that layer point to the terminal node $\terminalnode$. For some important DO problems, in particular for those considered in this paper,  it turns out to be useful to \textit{dynamically} select the decision variable to be associated with the next layer based on information regarding the current layer. In a DD compiled in the sketched top-down fashion, any pair of nodes in a layer has different states, that is, partial paths ending in the same state point to the same node.

\begin{algorithm}[htp]
    \caption{Build Layer procedure}
    \label{build_layer}
    \begin{algorithmic}[1] \baselineskip=11pt \relax
    \STATE \textit{BuildLayer} ($DP$, $\nodeSet_k$, $\bddwidth$, current variable)
    \FORALL{$u \in \nodeSet_{k}$} 
        \FORALL{$d \in \feasibledecisions(\state_u$)}
            \STATE $v =$ GetOrAddNode ($\nodeSet_{k},\transition(\state_u , d)$)
            \STATE AddArc (u,v,d)
        \ENDFOR
    \ENDFOR
    \IF{\(|\nodeSet_{k+1}|> \bddwidth \)}
        \STATE RelaxLayer/RestrictLayer ($\nodeSet_{k+1}$)
    \ENDIF
    \RETURN $\nodeSet_{k+1}$
    \end{algorithmic}
\end{algorithm}

In case of approximate DDs, after having created all nodes in a given layer, the size of that layer is reduced to $\bddwidth$ by removing or merging the nodes.  Nodes are merged by redirecting the incoming arcs of the nodes to be merged to a single merged node. In order to ensure that no feasible completions of any of the merged nodes is lost, one requires a problem-specific merge operator $\oplus$ for the states associated with the two nodes, see \cite{hooker2017job} for a discussion of the conditions a valid merge operator needs to satisfy. The described process of creating a single layer of an approximate DD with a maximum width  $\bddwidth$ is formalized in Algorithm~\ref{build_layer}. Given a DP formulation $DP$ (comprising the definition of the state space $\statespace$,  the functions $\feasibledecisions$, $\transition$ and $\transitioncost$), the (previous) layer $\nodeSet_{k}$ and the current decision variable, it creates and returns the next layer $\nodeSet_{k+1}$.  Observe that the \textit{RelaxLayer/RestrictLayer} step involves a heuristic node selection procedure that determines which nodes are merged or selected.

Algorithm~\ref{top-down} displays the pseudocode for full top-down compilation procedure for DD construction. The procedure takes a DP formulation $DP$, a DD $\DD$ containing only the root node and the maximum width $\bddwidth$. Calling the algorithm with an unlimited width $\bddwidth$ will yield an exact DD and depending on the operation performed in the \textit{BuildLayer} procedure, it will result in a restricted or relaxed DD. In order to allow for a dynamic variable selection, Algorithm~\ref{top-down} introduces the set \textit{unfixed} of variables that have not been considered so far in the compilation as well as the the generic procedure \textit{NextVariable} which chooses the next variable according to a given heuristic. We will provide examples for dynamic variable ordering heuristics in the next section. Note that in case of a static variable ordering strategy, \textit{NextVariable} simply returns the next variable  according to a pre-specified order.

\begin{algorithm}[htp]
    \caption{Top-Down DD Compilation}
    \label{top-down}
    \begin{algorithmic}[1] \baselineskip=11pt \relax
    \STATE CompileTopDown ($DP$, $\DD$, $\bddwidth$)
    \STATE \(\text{unfixed} = \text{set of all decision variables}\) 
    \FOR{$k = 1$ \TO  $n$}
        \STATE \(x_k=\) NextVariable (\(\nodeSet_{k}\), unfixed)
        \STATE unfixed \(= \text{unfixed}/\{x_k\}\)
        \STATE \(\nodeSet_{k+1}\) = \textit{BuildLayer} ($DP$, $\nodeSet_k$, $\bddwidth$, \(x_k\))
    \ENDFOR
    \RETURN $\DD$
    \end{algorithmic}
\end{algorithm}

\subsection{An Example: Decision Diagrams for the MWISP}

Here, we briefly introduce the Maximum Weighted Independent Set Problem (MWISP) and its DP formulation. We then illustrate how this DP formulation can be used to construct exact DD for the MWISP. Please note that a MIPS can be considered as a version of MWISP where all the weights are 1.

\paragraph{The Maximum Weighted Independent Set Problem.} Let $G = (V, E)$ be a weighted graph, where \(V=\{v_1, v_2, \cdots, v_n\}\) is the set of vertices and \(E\) is the set of edges. Moreover, every vertex \(v_i\) is associated with a positive integer weight \(w_i\). The Maximum Weighted Independent Set Problem (MWISP) asks for the subset $I \subseteq V$ with maximum weight such that no two vertices in \(I\) are connected via an edge, i.e. \(I=\{v \in V| (u,v) \notin E, \forall u \in I\}\).   

\textbf{Example.}
Fig.~\ref{example} shows an example of a weighted graph that will serve for illustration purposes. It shows a weighted graph $G$ with five vertices and their weights inside orange frames next to them. As can be easily verified, the optimal solution contains two vertices. i.e. \(I=\{v_3, v_4\}\).

\begin{figure}[htp]
    \centering
    \includegraphics[width=3cm]{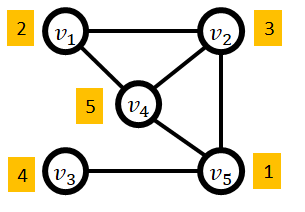}
    \caption{Example of a weighted graph $G$}
    \label{example}
\end{figure}

To allow for the creation of a DD for the MWISP, we formulate the MWISP in terms of a DP. To begin with, a state $S_u$ associated with a node $u$ in the DD corresponds to a set of vertices in \(V\) that are still available to be part of an independent set. The initial state $S_r$ associated with the root node $r$ thus corresponds to $V$, the terminal state corresponds to the empty set. Each layer $j$ in the DD is associated with the decision variable $x_j$ which consists in adding the $j$-th vertex $v_j$ (according to the chosen variable order) in the original graph $G$ to the independent set or not. Given a state $S_u$ and a decision $d(a)$ ($d=1$ means adding the vertex to the solution, $d=0$ not adding it) associated with arc $a=(u,u')$ emanating from node $u$, the state transition function $\transition_{j}(\state_{u},d(a))$ determines the state of node $u'$ in the next layer $j+1$ of the DD. Specifically, \( \transition_{j}(\state_{u},0) = \state_{u} \setminus \{v_j\} \), and \( \transition_j(\state_{u},1) = \state_{u} \setminus \{\Gamma(v_j)\}\) where \(\Gamma(v_j)\) is the set of vertices adjacent to \(v_j\) in \(V\). Note that if \(v_j \notin \state_{u} \), the decision $d=1$ is not feasible, and thus the DD will not contain an arc $a$ emanating from $u$ with $d(a)=1$. The reward function \(\transitioncost_{j}(\state_u,d)\) is \(\transitioncost_{j}(\state_u,0) = 0\) and \(\transitioncost_{j}(\state_u,1)=w_j\).

\textbf{Example (continued).}  Fig.~\ref{Exact_relax} (left side) shows an exact DD for the MWISP instance from Fig.~\ref{example} with the variables displayed in the picture. Every node of the DD is associated with the set of the remaining available vertices for consideration for the independent set and a small orange label next to it that shows the corresponding objective value. Each dashed and solid arc shows the assignment of values 0 and 1 to the corresponding decision variable, respectively. The exact DD in this example has a width of 3, and the longest r-t path is \([x_1=0, x_2=0, x_3=0, x_4=1, x_5=1]\), gives the optimum solution $\{v_3, v_4\}$ with value \(9\).

\begin{figure}[htp]
    \centering
    \includegraphics[width=12 cm]{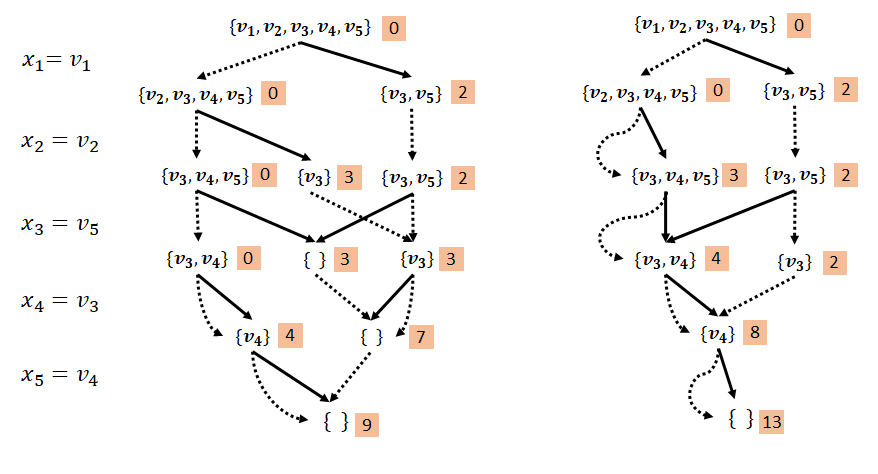}
    \caption{Exact (left side) and relaxed (right side) DDs for the MWISP of the example \ref{example}. Small orange labels next to each node are their partial objective value. Each dashed and solid arc shows the assignment of value 0 and 1 to the corresponding decision variable, respectively. The relaxed DD has \(\bddwidth=2\).}
    \label{Exact_relax}
\end{figure}

\paragraph{Relaxed Decision Diagrams for the MWISP}

A relaxed DD provides a discrete relaxation of the problem under consideration by over-approximating the solution space of the problem. As mentioned before, to guarantee a valid relaxation, the merge process must ensure that no feasible solution is lost. Merging two nodes $u$ and $u'$ into a node $M$ involves two steps: First, all incoming arcs to nodes \(u\) and \(u'\) must be redirected to merged node \(M\). Second, the state $S_M$ of the new node must be determined in a way that the solution space of the tail problem starting from \(M\) contains the tail problem solutions of both nodes $u$ and $u'$. As explained above, the state $S_M = S_u \oplus S_{u'}$, where $\oplus$ is a so-called merge operator. As discussed e.g. in \cite{Bergman2016DDbook}, a valid merge operator for MWISP and MISP is $\cup$, that is, $S_M$ is given by the union of the vertex sets (subsets of $V$ in the problem graph $G$) forming the states $S_u$ and $S_{u'}$. Fig.~\ref{Exact_relax} shows the relaxed DD for graph $G$ where \(\bddwidth=2\).

\subsection{The Dynamic Variable Ordering Heuristics}
The order of the decision variables according to which the DD is compiled heavily affects both the size of an exact DD, and the quality of the bounds from approximate DDs with a fixed maximum width. Different heuristics can be employed for variable ordering, e.g. a static variable ordering, that is, an ordering that is specified before the compilation of the DD and is independent of the configuration of the layers during the compilation process. However, it turns out that the best variable ordering heuristics for problems such as MWISP are dynamic, that is, they use information of the partially compiled DD to choose the next decision variable (in case of MWISP, the next vertex).

In the following, we briefly describe the dynamic variable ordering heuristic most commonly used for the MWISP in the literature, the so-called MIN (\textit{Minimum Number of States}) variable ordering (see e.g. \cite{Bergman2016DDbook}).

In the MIN variable ordering, vertices are assigned a value that corresponds to the number of times they appear in the states of the nodes in the current layer \(\nodeSet_k\). Then, the vertex exhibiting the minimum number of appearances is selected as the next vertex (variable) to be considered in the top-down compilation of the DD. The worst-case time complexity of performing this selection is \(O(\bddwidth\cdot |V|)\) per layer. We will use this heuristic applying our framework to the  MWISP.

\textbf{Example (continued).} The result from using the MIN variable ordering for the top-down compilation of a relaxed DD with \(\bddwidth=2\) for the example graph $G$ from Fig.~\ref{example} is shown in Fig.~\ref{Exact_relax} (right side, with upper (dual) bound 13).

In the next section, we propose a novel clustering-based variable ordering framework that applies this dynamic ordering heuristic (MIN) to subsets of variables obtained by clustering. 

\section{A novel Framework for Variable Ordering}

The main motivation for our clustering framework is that while in general, one strives for the best-possible variable ordering in order to achieve strong dual bounds, as exemplified above, some of the most successful variable ordering heuristics come with a non-negligible time complexity. When employed in an exact algorithm such as DD-based branch-and-bound, however, there is always a trade-off between the quality of the bounds obtained from a relaxed DD and the time required to compile it.

DD-based branch-and-bound uses relaxed DDs as its base search tree and also as an algorithm to obtain dual bounds on the subproblems. Therefore, it might be beneficial to speed up the compilation of the relaxed DDs without heavily impairing the quality of their bounds. As mentioned before, the complexity of MIN is \(O(\bddwidth\cdot |V|)\) per layer. Suppose that instead of applying MIN to all the variables, we apply it only to a subset of the variables in every layer. This indeed will decrease the complexity and hence the run-time. Therefore, if one can control the decrease of the quality of the dual bounds to be in moderate relation with the gains in run-time, it could be the case that overall solution time of the DD-based branch-and-bound will decrease. Motivated by this hypothesis, we came up with the idea of grouping the set of variables into subsets of vertices to be fed to MIN variable ordering heuristic. To perform such a grouping, we decided to employ clustering algorithms, e.g. k-means clustering. Assume that the variables are clustered into \(n_c\) clusters and let \(\frac{|V|}{n_c}\) be the size of the clusters (approximately); then the complexity changes to \(O(\bddwidth\cdot \frac{|V|}{n_c})\) per layer.

In this framework, the variables are first clustered into groups according to problem-specific features. The construction process then proceeds in two ways, which are described in the following subsections. Afterwards, we discuss how the number of clusters can be determined.

\subsection{Cluster-by-Cluster (CbC)}

In this strategy, which is called \textit{Cluster-by-Cluster (CbC)} , the clusters are sorted based on some criteria before executing the top-down compilation. In other words, this sorting step is operated between the clustering step and the rest of the construction process.
Algorithm~\ref{top-down_gbg} shows the pseudocode of the top-down compilation in which the proposed technique is adapted to compile a relaxed decision diagram. First, the variables are clustered into \(n_c\) clusters, i.e. line 2. Then, in line 3, the clusters are sorted in non-decreasing order according to some heuristics such as their sum of the vertices' weights for MWISP. The for loop in line 5 goes into the clusters in the sorted order, and the while loop starting at line 6 feeds variables of the cluster to a variable ordering heuristic (line 7). Then, the corresponding DD layer is built and checked if the width of the layer exceeds the given maximum width \(\bddwidth\) and if it does, a relaxation is performed on the layer in which nodes will be merged until the \(\bddwidth\) is respected (procedure \textit{BuildLayer} at line 9).

\begin{algorithm}[htp]
    \caption{Top-Down DD Compilation using clustering-based Variable Ordering Framework; Cluster-by-Cluster}
    \label{top-down_gbg}
    \begin{algorithmic}[1] \baselineskip=11pt \relax
    \STATE CompileTopDown ($DP$, $\DD$, $\bddwidth$, $n_C$)
    \STATE clusters = cluster variables (vertices) into \(n_c\) clusters
    \STATE \(SC\) = vector of clusters sorted by the sum of the vertex weights in each cluster
    \STATE counter = \(2\)
    \FOR{$k = 1$ \TO  $n_c$}
        \WHILE{\(|SC[k]| \ge 1\)}
            \STATE variable \(=\) NextVariable (\(\nodeSet_{k}\), \(SC[k]\))
            \STATE \(SC[k] = SC[k] \setminus \{\text{variable}\}\)
            \STATE \(\nodeSet_{counter}\) = \textit{BuildLayer} ($DP$, $\nodeSet_{counter-1}$, $\bddwidth$, variable)
            \STATE counter += \(1\)
        \ENDWHILE
    \ENDFOR
    \RETURN $\DD$
    \end{algorithmic}
\end{algorithm} 

\subsection{Pick-and-Sort (PaS)}

Another strategy that can be employed in our clustering-based variable ordering framework is to not order the clusters but rather alternate between them to pick one variable from every cluster and then sort the selected variables and build their corresponding layers in the sorted order. This procedure is repeated until all variables of the problem are considered. We call this strategy \textit{Pick-and-Sort (PaS)}. This strategy differs from the previous one, i.e. Cluster-by-Cluster, in the sense that it does not exhaust the clusters one by one; in every step one variable from every cluster is being taken into account. Intuitively, this approach can benefit from multiple ways for sorting the selected variables; (i) a problem specific criterion (PaS), e.g. weight of the vertices in case of MWISP, (ii) the (problem-agnostic) information obtained from the variable ordering heuristic that is being used within the framework (PaS-VO). The motivation behind this strategy is to improve the previous strategy, i.e. Cluster-by-Cluster, in a sense that it tries to lower the loss of quality of the relaxed DD bounds by decreasing the restriction on the order of variables and giving it more freedom that could result in a more similar ordering to the classical variable ordering. 
The pseudocode of the corresponding top-down compilation is presented in Algorithm~\ref{top-down_ps}. After grouping the variables into \(n_c\) clusters (line 2), the algorithm enters the main loop in which it is controlled that all \(n+1\) layers of the DD corresponding to \(n\) variables be compiled (starting from lines 4). Then, a variable from each cluster is selected and is pushed (appended) into a vector called \textit{variables} (lines 5 - 9). At line 10, the algorithms sorts the vector \textit{variables} according to the two different ways we described earlier. In the first way, the selected variables are sorted according to a problem specific criterion. In the second way, the selected variables are sorted based on the values that were assigned to each variable in the \textit{NextVariable} heuristic that was done in line 7. In the next step, a for loop is used to loop over the sorted vector of variables (line 11) and build the corresponding layers and in the case of exceeding the maximum width \(\bddwidth\) of a layer, a relaxation of the layer takes place to ensure that \(\bddwidth\) is respected, that all is done via the procedure \textit{BuildLayer} at lines 12.

\begin{algorithm}[htp]
    \caption{Top-Down DD Compilation using clustering-based Variable Ordering Framework; Pick-and-Sort}
    \label{top-down_ps}
    \begin{algorithmic}[1] \baselineskip=11pt \relax
    \STATE CompileTopDown ($DP$, $\DD$, $\bddwidth$, $n_C$)
    \STATE clusters = cluster variables (vertices) into \(n_c\) clusters
    \STATE counter = \(1\)
    \WHILE{counter $\le n $}
        \STATE variables \(= [~]\)
        \FOR{cluster in clusters}
            \STATE variable \(=\) NextVariable (\(\nodeSet_{k}\), cluster )
            \STATE cluster \( = \text{cluster} \setminus \{\text{variable}\}\)
            \STATE Append\((\text{variables, variable})\)
        \ENDFOR
        \STATE sort variables by either (i) a problem-specific criterion, or (ii) their assigned value according to NextVariable (\(\nodeSet_{k}\), cluster ) heuristic
        \FOR{variable \(\in\) variables}
            \STATE \(\nodeSet_{counter}\) = \textit{BuildLayer} ($DP$, $\nodeSet_{counter-1}$, $\bddwidth$, variable)
            \STATE counter +=\(1\)
        \ENDFOR
    \ENDWHILE
    \RETURN $\DD$
    \end{algorithmic}
\end{algorithm} 
\subsection{Determining the Number of Clusters}

In this section, we present two policies for determining the number of clusters, which is a key parameter of our framework. These policies are based on a Lemma on the growth of the size the DD layers for an MWISP instance. We start with the following definition that we will use for developing the mentioned Lemma. 

\begin{definition}[Reduced decision diagram]
    A decision diagram in which no two states are equivalent, meaning that they exhibit the same set of completions, is called a \textit{reduced decision diagram}.  
\end{definition}

In other words, the set of solutions of the subproblems corresponding to equivalent states are the same. Deciding whether two states are equivalent when compiling a DD in a top-down fashion is an NP-hard problem for some problems, e.g. 0/1 Knapsack. However, there exist problems for which the equivalence of a given pair of states can be checked in polynomial time, e.g. MWISP. Let decision diagram \(\DD\) be a decision diagram compiled for an instance of a MWISP, then two states \(S\) and \(S'\) are equivalent if and only if \(S=S'\). This can be checked with a simple state comparison in the top-down compilation of \(\DD\) for MWISP. Moreover, any pair of equivalent states in a DD can be merged together without causing any over-approximation (under-approximation) of the solution.

\begin{lemma}[Equivalence in MWISP]\label{eq_misp}
    Let \(\DD\) be a decision diagram for MWISP (MISP) and let \(S\) and \(S'\) be two states in a layer of \(\DD\). States \(S\) and \(S'\) are equivalent if and only if \(S=S'\). 
\end{lemma}

\begin{proof}
    Suppose that \(S=S'\). It means that they both contain the same set of vertices of the original graph, therefore, they correspond to the same induced subgraph of the original graph. Hence, the set of solutions, i.e. set of independent sets, for both of them are the same. Therefore, states \(S\) and \(S'\) are equivalent.

    Now assume that the states \(S\) and \(S'\) are equivalent. This means that the subgraphs induced by their vertices are the same. This only happens if the set of vertices for both states are the same. Hence, \(S=S'\). \(\hfill \square\)
    
\end{proof}

Using Lemma \ref{eq_misp}, it is easy to see that the size of the layers of a DD for MWISP in bottom-up view grows exponentially in the power of 2 with the terminal layer containing only one state which is an empty set. For simplicity, one can assume that bottom-up traversing the DD layers' states resembles constructing all possible subsets of the set of vertices at each layer starting from the terminal layer (terminal state) and add a new vertex into consideration in every layer. 

\begin{lemma}[Bottom-up growth of MWISP's DD layers]\label{bu_growth}
    Let \(\DD\) be a decision diagram built for an instance of MWISP where there exists no equivalent states. The size of \(\DD\) in bottom-up view in the worst case grows exponentially in power of 2, i.e. \(|\nodeSet_{n+1-z}| \le 2^z\), where \(z \in \{1, 2, \cdots, n\}\).
\end{lemma}

\begin{proof}
    We just need to show that the size of every layer \(\nodeSet_{n+1-z}\) in decision diagram \(\DD\) of an instance of a MWISP, where \(z \in \{1,2,  \cdots, n\}\), is bounded by \(2^z\), i.e. it contains at most \(2^z\) states corresponding to all subsets of a set of vertices with \(z\) members, i.e. \(\{x_n, x_{n-1}. \cdots, x_{n+1-z}\}\). To do this, assume that we are given the ordered list of variables and we want to construct the corresponding DD layers in a bottom-up fashion starting from the last layer, \(\nodeSet_{n+1}\), which contains the terminal state \(\{\}\). At every layer \(\nodeSet_{n+1-z}\) there are still \(z\) variables in the ordered list of variables in top-down compilation of \(\DD\) that has not been considered yet. Recall that in \(\DD\) every state represents the set of available vertices for consideration for independent set. Therefore, if we were to guess the corresponding states for this layer, in the worst case, we would have to construct all possible subsets that can be formed using these \(z\) variables in order to be able to include all possible subsets of available vertices for consideration in top-down fashion. This means that this layer can contain at most \(2^z\) states, i.e. let \(z=1\) then layer \(\nodeSet_{n}\) contains only two states \(\{\}\) and \(\{x_{n}\}\), let \(z=2\), i.e. \(\{x_{n-1},x_{n}\}\), then there are at most \(4\) possible subsets representing the sets of available vertices for consideration, i.e. \(\{\}, \{x_n\}, \{x_{n-1}\}, \{x_n, x_{n-1}\}\). Now assume that the number of states at layer \(\nodeSet_{n+1-z}\) is more than \(2^z\), this implies that at layer \(\nodeSet_{n+1-z}\), there exist more than \(z\) vertices that are not yet considered which contradicts the steps of a top-down compilation of \(\DD\). \(\hfill \square\)   
    
\end{proof}

Knowing that a DD \(\DD\)  represents a MWISP (MISP) instance, one can approximate the size of \(\DD\) in the worst case. Assuming that the value of the given maximum width is fixed during a DD-based branch-and-bound, we can compute a value based on the given maximum width for the size of subproblems within the DD-based branch-and-bound for which the resulting DD will be exact no matter how their corresponding variables are ordered, meaning that the width of their corresponding DDs will not exceed the given maximum width. 
This is particularly useful in a DD-based branch-and-bound of the MWISP, since we can simply avoid using any sophisticated variable ordering heuristic for such subproblems. Therefore, we present the following lemma on approximating the size of a DD for an instance of MWISP.

\begin{lemma}[Size of MWISP's DD]\label{misp_dd_size}
    Let \(\DD\) be a decision diagram compiled for an instance of MWISP (MISP) with \(n\) vertices. Then the width of the largest layer of \(\DD\) will not exceed \(2^{\lfloor n/2 \rfloor}\).
\end{lemma}

\begin{proof}
    We assume a DD \(\DD\) for MWISP in which every pair of equivalent states is already merged, that is, \(\DD\) does not contain equivalent states. Regardless of containing equivalent states, if we consider the width of layers from the layer containing only the root (initial) state (initial layer), in the worst case the width of the layers grows exponentially by powers of 2, i.e. \(2^0, 2^1, 2^2, \cdots, 2^{k-1}, \cdots, 2^n\), where \(k\)th layer contains \(2^{k-1}\) states. This shows the growth of the width of the layers in top-down view. Now consider the sequence showing the growth in bottom-up view, that is \(2^{0}, 2^{1}, 2^{2},  \cdots, 2^{z}, \cdots, 2^n\), where \(z\)th layer contains \(2^{z-1}\) states, and \(z = (n + 1 - k) \in \{n+1, n, \cdots, 1\}\) since \(k \in \{1, 2, \cdots, n+1\}\) (Fig. \ref{growth}).  
    
    Let \(|\nodeSet_{k+1}|\) be the width of the layer \(k+1\), considering the two sequences that correspond to the growth of the width of the layers from top-down and from bottom-up view in \(\DD\), the following holds for any \(k \in \{1, 2, \cdots,n-1\}\) (\(|\nodeSet_{1}|=|\nodeSet_{n+1}|=1\)):

    \begin{equation}
        \begin{cases}
            |\nodeSet_{k+1}| \le 2^{k}, \text{ from the sequence of top-down view}
            \textbf{ }\\
            |\nodeSet_{k+1}| \le 2^{n-k}, \text{ from the sequence of bottom-up view} 
        \end{cases}\label{width}
    \end{equation}

    This means that at some point the width of the layers stops increasing and in fact starts decreasing; it is an implication of the inequalities \ref{width}. Now it just remains to find the \(k\) for which \(|\nodeSet_{k+1}|\) attains its maximum. And this happens when \(2^{k} = 2^{n-k}\). Now consider two cases; (i) \(n\) is an even number, and (ii) \(n\) is an odd number. In the first case, the layer with index \(k+1=(n/2)+1\) has the highest width and it is \(|\nodeSet_{k+1}|=2^{n/2}\). In the second case, there are two layers with highest width and their indices are \(k+1=\lfloor n/2\rfloor+1\) and \(k+2=\lfloor n/2\rfloor+2\) and they both have a width of \(|\nodeSet_{k+1}|=|\nodeSet_{k+2}|=2^{\lfloor n/2 \rfloor}\). \(\hfill \square\)     
\end{proof}

\begin{figure}[htp]
    \centering
    \includegraphics[width=12cm]{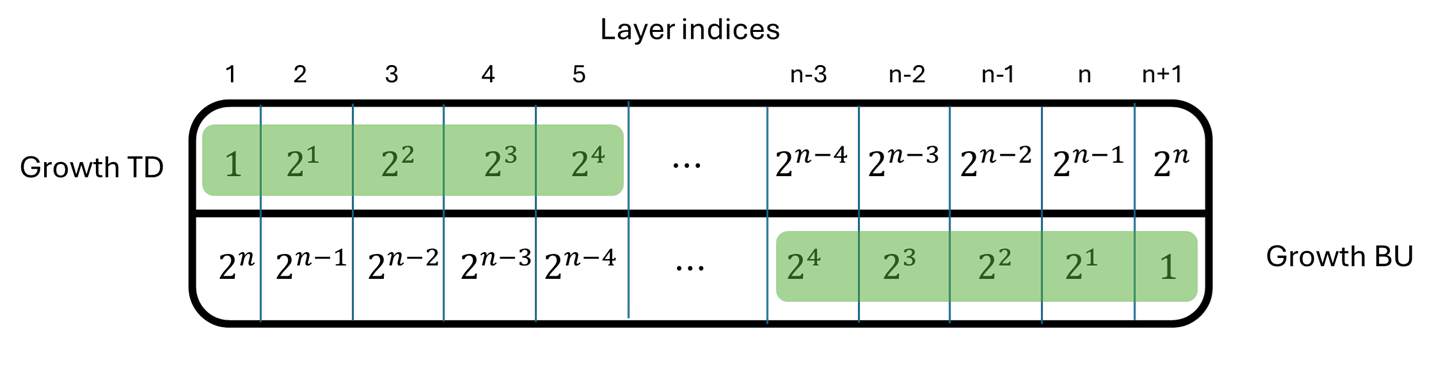}
    \caption{Growth of the widths of the layers of a DD for a MISP instance with \(n\) vertices; first row corresponds to the sequence related to top-down (TD) view starting from initial layer, second row corresponds to the sequence related to bottom-up (BU) view starting from terminal layer. The widths highlighted with green box are the highest possible width for the corresponding layers. E.g. \(|\nodeSet_{5}| \le 2^4 \& |\nodeSet_{5}| \le 2^{n-4} \Rightarrow |\nodeSet_5| \le 2^4 \).}
    \label{growth}
\end{figure}

Lemma \ref{misp_dd_size} implies that the width of no layer in a DD \(\DD\) for a MWISP (MISP) with the number of variables \(n \le 2\cdot \lceil \log_2 \bddwidth \rceil\), where \(\bddwidth\) is the given maximum width, will exceed the given maximum width. Therefore, such a DD \(\DD\) is always exact. This is particularly interesting and useful to know when dealing with a DD-based branch-and-bound for solving a MWISP (MISP). One of its implications is that if the maximum width of restricted and relaxed DDs within the branch and bound are fixed to \(\bddwidth\), then any subproblems within the branch and bound that contains less than \(2\cdot \lceil \log_2 \bddwidth \rceil\) vertices would be solved to optimality in the first try as the created restricted and/or relaxed DD for it actually turns out to be exact even in the worst case. 

Therefore, for such subproblems, no clustering of variables is used. Another use of this property is that the number of clusters \(n_c\) cannot be fixed to a value greater than the size of the largest of such subproblems, i.e. \(n_c < 2\cdot \lceil \log_2 \bddwidth \rceil\).
Intuitively, the framework in this paper would benefit more from having a dynamically adapted number of clusters rather than having a fixed number of clusters for the entire DD-based branch-and-bound. However, we can consider both policies, each with a fixed condition on the number of clusters. Therefore, we came up with the following policies on the choice of the number of clusters:

\begin{itemize}
    \item \textbf{\textit{fixed}}; in this case, we fix the number of clusters \(n_c\) on 2, i.e. \(n_c = 2\). One reason for this choice is that for subproblems with size \(n'\) where \(2\cdot \lceil \log_2 \bddwidth \rceil < n' < 4\cdot \lceil \log_2 \bddwidth \rceil\) it does not make sense to cluster the variables into more than 2 clusters. Moreover, 2 clusters still would make sense for larger subproblems as it would still decrease the running time of top-down compilation of relaxed DDs and the decrease in quality of relaxed DDs would be moderate. If one chooses a number greater than 2, as soon as the size of subproblems becomes smaller than \(4\cdot \lceil \log_2 \bddwidth \rceil\), \(n_c\) should be changed to 2, therefore, it becomes an adaptive policy which we are considering as the second policy.
    \item \textbf{\textit{adaptive}}; in this case, we would use the following formula: \(n_c= Max(\frac{R}{2},1)\) where \( R = \Big\lfloor \frac{n'}{2\cdot \lceil \log_2 \bddwidth \rceil} \Big\rfloor \) and \(n'\) is the size of the subproblem being solved in the DD-based branch-and-bound. However, one could also think of other adaptive formulas, e.g. by designing it in a way that the approximate size of each cluster in every step does not exceed a given specific percentage of the subproblem size. 
\end{itemize}

\section{Computational Results}

In this section, we present the results of computational experiments with different policies for the choice of the number of clusters; \textit{adaptive} and \textit{fixed}, when creating relaxed DDs within a classical DD-based branch-and-bound algorithm \cite{bergman2016branch} to solve the instances of MWISP to optimality. 
A simple k-means clustering algorithm has been used for clustering the vertices in which each vertex \(v_i\) is associated with a feature vector for MWISP, e.g. \([dg(v_i), w(v_i)]\), where \(dg(v_i)\) and \( w(v_i)\) are the degree and the weight of the vertex \(v_i\), respectively. Moreover, to perform the experiments we used SortObj (SO) node selection heuristic; a generic and most commonly used node selection heuristic in the literature for creating restricted and relaxed DDs (in a relaxed DD nodes are sorted according to their objective values, then the best \(\bddwidth-1\) nodes are retained, and then the rest of the nodes are merged into a single merged node). The maximum width of restricted and relaxed DDs are set to \(\bddwidth=100\) for the entire DD-based branch-and-bound. Moreover, the variables/vertices in restricted DDs are sorted in a non-increasing order based on their weights. 

For performing the experiments, we took the sets of instances (180 instances in total) from \cite{nafar2024strengthening}, each of which contains 20 graphs with 100 vertices. The instance sets differ with respect to the graph density which ranges from 0.9 to 0.1 (instances along with the code and the detailed results are available here \url{https://github.com/mnafar/Clustering-Based-Variable-Ordering-Framework-for-Relaxed-DDs} ). For the weights of the vertices in the graphs, we just assigned the value \((i \textit{ mod } 100)+1\) to every vertex \(v_i\). We implemented all approaches in Pluto Notebook (Julia programming language) and ran the experiments on a Windows machine with 32GB RAM and an Intel(R) Core(TM) Ultra 7 164U processor with 2.10 GHz. 

We report the performance of our proposed approaches, i.e. CbC, PaS, and PaS-VO in Table \ref{sol_time}, all the values are based on seconds. We evaluated the performances of our approaches in terms of the solution times in comparison to the baseline, i.e. using MIN variable ordering with no clustering.

\setlength{\tabcolsep}{9pt}
\begin{table}[htp]
    \centering   
    \scalebox{.9}{
  \begin{tabular}{@{}c|c|cc|cc|cc@{}} 
 Densities & \multicolumn{7}{c}{ Approaches }  \\ \hline   

  & \multicolumn{1}{c}{ Baseline } & \multicolumn{2}{c}{ CbC } & \multicolumn{2}{c}{ PaS-VO } & \multicolumn{2}{c}{ PaS } \\ \hline

  &  & \multicolumn{1}{c}{ fixed }  &  \multicolumn{1}{c}{ adaptive } & \multicolumn{1}{c}{ fixed }  &  \multicolumn{1}{c}{ adaptive }& \multicolumn{1}{c}{ fixed }  &  \multicolumn{1}{c}{ adaptive }  \\   \hline

0.1 & 196.21 & 573.58 & 535.29 & 237.61 & 364.32 & 203.80 & 278.40  \\ 
0.2 & 49.16 & 58.37 & 59.01 & 54.28 & 66.87 & \textbf{48.79} & 55.69  \\
0.3 & 14.87 & 13.29 & 13.88 & 13.88 & 16.18 & \textbf{12.61} & 14.00 \\
0.4 & 5.20 & 3.55  & \textbf{3.50 } & 4.85 & 5.17 & 4.32 & 4.59 \\
0.5 & 2.53 & 1.64  & \textbf{1.56 } & 2.29 & 2.15 & 2.14 & 2.04 \\
0.6 & 1.23 & 0.94  & \textbf{0.86 } & 1.17 & 1.12 & 1.02 & 1.10 \\
0.7 & 0.74 & 0.52  & \textbf{0.50 } & 0.66 & 0.65 & 0.63 & 0.63 \\
0.8 & 0.57 & 0.39  & \textbf{0.35 } & 0.50 & 0.47 & 0.44 & 0.46 \\
0.9 & 0.37 & 0.26  & \textbf{0.23 } & 0.34 & 0.29 & 0.31 & 0.29 \\ \hline

\end{tabular}%
}
\caption{Solution times for MWISP using different variable ordering strategies, i.e. Base-line (no clustering), Cluster-by-Cluster, Pick-and-Sort, and Pick-and-Sort VO, along with different policies for determination of the number of clusters, i.e. fixed and adaptive.}
\label{sol_time} 
\end{table}

It is evident from the table that all our proposed clustering-based variable ordering frameworks with both policies for setting the number of clusters improve the solution time over almost all instances except instances with density 0.1. More precisely, using the Cluster-by-Cluster strategy within our proposed framework with the number of clusters being set using an adaptive policy significantly reduces the solution time of the DD-based branch-and-bound algorithm across all instances for graphs with densities 0.9 down until 0.4. Moreover, for graph instances with densities 0.3 and 0.2, strategy PaS with having a fixed number of clusters has the best performance across the proposed approaches and is superior to the baseline. Although this approach does not perform better than the baseline for graphs with density 0.1, its performance is relatively close to the baseline. All in all, for graphs with density more than 0.3 all our proposed approaches are superior to the baseline, and for densities 0.3, 0.2, and 0.1 PaS would be the best choice among the proposed approaches,    yielding better results than the baseline in most cases  and being close for one instance class.

\section{Conclusion}

Variable ordering is one of the key heuristic decisions within the DD compilation process, which widely affects the quality of the obtained solution. In this work, we introduced a novel and generic clustering-based variable ordering framework, along with two strategies for selecting variables: (1) Cluster-by-Cluster, (2) Pick-and-Sort. Moreover, motivated by a theoretical study on the size of DDs for MWISP we proposed two different policies for choosing the number of clusters within our framework; (1) adaptive, (2) fixed. Integrated into a DD-based branch-and-bound algorithm, our framework significantly reduces DD-based branch-and-bound solution times for instances of the Maximum Weighted Independent Set Problem (MWISP). These results demonstrate that even simple machine learning techniques can meaningfully enhance exact algorithms. While we only considered MWISP, we believe that the proposed framework, policies, and theory can be generalized and therefore suggests promising potential for broader application to other combinatorial optimization problems. Moreover, depending on the problem one can define different feature vectors for clustering the variables, even for MWISP there are other ways to define features for variables. Furthermore, as long as it goes for sorting criteria one can use different criteria depending on the problem and their preferences. However, PaS-VO is somewhat resistant to a different sorting heuristic as it uses the information gained from the underlying variable ordering heuristic within the framework. In a future work, it would be interesting to define a concept of similarity on the set of variables and then represent it via a simple undirected graph where edges between variables are present if they are similar enough. And then use graph clustering algorithms for clustering of the variables. This would be a highly generic approach that can give the user the chance to define the similarity measurement themselves.

%
%
%
\bibliographystyle{splncs04}
\bibliography{biblio}
%





\end{document}